\documentclass{article}
\usepackage{arxiv}

\usepackage[hyphens]{url}
\usepackage{graphicx}
\graphicspath{{figures/}}
\usepackage[small,bf]{caption}
\usepackage{booktabs}
\usepackage{amsmath}
\usepackage{amssymb}
\usepackage{algorithm}
\usepackage{algorithmic}
\usepackage{placeins}
\usepackage[round]{natbib}
\renewcommand{\cite}{\citep}
\usepackage[colorlinks=true,linkcolor=black,citecolor=blue,urlcolor=blue]{hyperref}
\newtheorem{proposition}{Proposition}

\title{Evidential Rule Learning for Interpretable Classification\\with Abstention}

\author{%
  \begin{tabular}[t]{@{}l@{}}
    \textbf{Javier Fumanal-Idocin}\\
    School of Computer Science\\
    and Electronic Engineering\\
    University of Essex, Colchester CO4 3SQ, UK\\
    \texttt{j.fumanal-idocin@essex.ac.uk}
  \end{tabular}%
  \hspace{1.5em}%
  \begin{tabular}[t]{@{}l@{}}
    \textbf{Javier Andreu-Perez}\\
    School of Computer Science\\
    and Electronic Engineering\\
    University of Essex, Colchester CO4 3SQ, UK\\
    \texttt{j.andreu-perez@essex.ac.uk}
  \end{tabular}%
}

\renewcommand{\shorttitle}{Evidential Rule Learning with Abstention}
\renewcommand{\headeright}{}

\begin{document}

\maketitle

\begin{abstract}
 Interpretable classification often requires more than accurate predictions for real-life deployment: models should be transparent about the evidence behind their decisions and abstain when they cannot decide reliably. We introduce Fast Evidential Rule Learning (FERL), a method that learns interpretable, accurate fuzzy rule models whose outputs are evidential. Unlike post-hoc calibration, FERL's belief, plausibility, and abstention capabilities arise directly from the fuzzy memberships in a single deterministic pass, with no auxiliary head, held-out set, or repeated inference. Our theoretical analysis further shows that FERL is Lipschitz stable, which means that its evidential outputs vary smoothly with the input. Against state-of-the-art rule learners, FERL is statistically significantly more accurate across a 30 tabular-dataset benchmark ($+2.6\%$ average accuracy over the second best). Its native set predictions attain the best utility-discounted accuracy among credal classifiers ($u_{65}/u_{80}=0.80/0.83$ vs.\ $0.79/0.80$ for the naive credal classifier), at higher set coverage ($0.92$ vs.\ $\le0.82$). FERL also matches dedicated out-of-distribution detectors on tabular near-OOD detection ($77.7$ vs.\ $77.4$ AUROC for the strongest baseline). Under detector-class-disjoint concept-bottleneck evaluation, its it is within $2.3$ AUROC points of the strongest dedicated detector on both CUB and AwA2, while attaining the best AwA2 AUPR-Out ($68.3$) and novel-class rejection ($57.2$), while being able to name which attributes are anomalous.
\end{abstract}


\section{Introduction}

Decision trees and rule lists remain the reference models whenever a
prediction has to be justified to a human: a leaf is a conjunction of
conditions, and the path that reaches it, the explanation
\cite{rudin2019stop}. This transparency, however, usually comes at a price.
Tree-based classifiers such as CART \cite{breiman1984cart} and C4.5
\cite{quinlan1993c45} must grow large to be accurate, and a single
leaf returns a hard class label or local class distribution with only a limited
account of how much evidence supports the decision. This makes standard tree probability estimates biased towards overconfident values \citep{provost2003tree}. Rules are also uninformative when learned directly over unstructured inputs such as image pixels, so they require neuro-symbolic approaches, e.g., concept bottleneck models (CBMs) \citep{koh2020concept}. Because of these problems, post-hoc explanation methods such as LIME \citep{ribeiro2016lime}, SHAP \citep{lundberg2017shap}, and Grad-CAM \citep{selvaraju2020grad} remain an alternative for building explanations, but they also have problems: they explain an already existing predictor and can be unfaithful to its actual decision \citep{adebayo2018sanity, taimeskhanov2024cam} and they might give a poor estimation of the global reasoning process of the explained model \citep{rudin2019stop}. Which is why interpretable-by-design models are preferable when possible.

For high-stakes domains interpretability alone is not enough: a model must also report realistic, reliable confidence in each decision \citep{hullermeier2021aleatoric}. A popular approach to do this is to calibrate the model predictions, or use a score that quantifies the uncertainty in each decision. Post-hoc calibrators can be model-agnostic, like temperature scaling \cite{guo2017calibration}, Dirichlet calibration \citep{NEURIPS2019_8ca01ea9}, isotonic regression \cite{zadrozny2002isotonic} and conformal predictions \cite{angelopoulos2023gentle}. There are also specific alternatives for trees, such as per-leaf calibration \citep{leathart2017probability, johansson2019calibrating}, that adjust the prediction probability, although not the amount of evidence supporting it \citep{amini2020deep}. These techniques also tend to degrade under distribution shift \citep{ovadia2019can}, which is problematic for real-life deployment. 
Reliability can also be built into the prediction. Sampling-based methods such as MC-dropout \citep{gal2016dropout} and deep ensembles \citep{lakshminarayanan2017ensembles} estimate uncertainty but require repeated inference, while deterministic uncertainty methods (DUMs) capture \emph{epistemic} uncertainty in a single pass, using distance-aware \citep{mukhoti2023ddu} or evidential \citep{amini2020deep} scores. The evidential approaches come from belief-function theory \cite{shafer1976evidence,smets1994tbm}. Popular approaches in this direction are evidential deep learning \cite{sensoy2018edl,amini2020deep} and Dempster--Shafer (DS) neural networks \cite{denoeux2019logistic}. Evidential fuzzy rule classifiers use the possibility semantics of fuzzy memberships \cite{dubois1988possibility} as evidence \citep{shiraishi2025evidential}. There are, however, some problems in existing DUMs. They need spectral fine-tuning of the neural architecture to avoid feature collapse, and because the score depends on the geometry of a learned representation, they only reliably detect geometric novelty. This means that they tend to assign deceptively low uncertainty to a near-OOD input that overlaps the training manifold. Besides, they report a scalar with no symbolic account of \emph{why} a case is doubtful.

In this work, we introduce Fast Evidential Rule Learning (FERL), a fuzzy rule-tree learner where  each node emits a DS mass function in which the firing strength of a rule becomes the mass placed on its class distribution, and whatever does not fire becomes mass on ignorance. This makes every prediction carry a point label, a graded confidence, and a set-valued prediction. FERL is, to our knowledge, the first model that is at once an interpretable rule learner, a credal classifier with abstention and an OOD detector, in a single deterministic pass. This makes FERL a compelling way to add interpretability, reliability, and principled abstention to tabular classification and to neuro-symbolic pipelines.

As a summary, our contributions are as follows:

\begin{itemize}
	\item \textbf{A native evidential and accurate rule learner.} We introduce FERL, a fuzzy rule tree whose DS evidence is read directly off the fuzzy memberships, computing belief, 	plausibility, credal prediction sets, and abstention in a single pass, with no separate calibration model, held-out set, or repeated forward passes.
	\item \textbf{A closed-form account of construction and evidential prediction.} We analytically characterise both the learned fuzzy tree and its evidence sources: fuzzy splits weakly decrease Gini impurity, conserve in-support routing mass, and give a local stability certificate controlled by their band widths. The induced mass functions admit a closed form, interval-dominance prediction reduces to an ignorance-margin rule, and repeated Dempster combination of nested nodes monotonically shrinks ignorance.
	\item \textbf{Broad empirical evaluation.} We benchmark FERL on 30 tabular classification datasets spanning diverse sizes, domains, and feature dimensions. We also benchmark FERL on 2 image datasets with annotated concepts within a concept-bottleneck image benchmark. In this setting, FERL operates on concept vectors rather than raw pixels, so its rules and abstention decisions are expressed in human-understandable attributes. FERL OOD performance matches dedicated OOD detectors while naming the anomalous attributes, and stays informative under off-support distribution shift, where it becomes more prone to abstain as its prediction sets widen.
\end{itemize}

\section{Related Work}

\paragraph{Interpretable trees and rule learners.}
CART \cite{breiman1984cart} and C4.5 \cite{quinlan1993c45} are the canonical
greedy trees, and RIPPER \cite{cohen1995ripper} is another classic algorithm that produces compact ordered rule lists. Fuzzy rule learners soften crisp thresholds using fuzzy logic:
decision trees replace hard splits with membership functions
\cite{olaru2003fuzzy, fumanal2024ex, huhn2009furia}. Recent work uses gradient-based and neural optimisation for rule-learning \citep{omran2018scalable, qiao2021learning, wang2021scalable, wang2023learning, dierckx2023rlnet, fumanal2025compact, xu2026neural}. All of these approaches rely heavily on the Straight-Through estimator \citep{bengio2013estimating} to mitigate the problem of non-differentiable argmax operations in rule-based reasoning. Modern non-neural literature also focuses on scalability: FIGS \cite{tan2022figs} grows a sum of shallow trees to stay as small as possible, and rule sampling \citep{pellegrina2024scalable} tries to obtain the smallest size with guarantees. FERL closes both gaps: it learns accurate classifiers across different rulebase-complexity levels and trains efficiently to handle scale to large datasets.

\paragraph{Evidential and credal classification.}
Dempster--Shafer theory \cite{dempster1967upper,shafer1976evidence} models partial and conflicting evidence with belief functions, with a long history in classification, from the evidential $k$-NN rule \cite{denoeux1995knn} to evidential neural networks \cite{sensoy2018edl}. In deep learning \cite{ulmer2023survey}, such methods predict Dirichlet (or belief) parameters in a single pass, grounded in subjective logic \cite{josang2016subjective}. Quantified ignorance supports abstention above a threshold \citep{xin2021art, mao2024predictor}. FERL shares this single-pass procedure but computes evidence using interpretable fuzzy rules rather than a black-box network, following the established connection between graded membership and belief functions: a source firing with strength $\phi$ induces a Bayesian mass discounted by $1-\phi$ \cite{smets1994tbm}. 

Credal classifiers reason with sets of probabilities: the naive credal classifier \cite{corani2008ncc} and credal decision trees built on the imprecise Dirichlet model \cite{abellan2003credaltree} return non-dominated classes, evaluated by utility-discounted accuracy \cite{zaffalon2012evaluating}. Recent work uses credal sets to quantify decision uncertainty as well \citep{lienen2021credal, wang2025credal}. Unlike existing tree-based credal classifiers, FERL produces set-valued outputs natively from fuzzy memberships rather than imprecise leaf-frequency estimates, mitigating the miscalibration of tree-based inference.

\section{Method}

\paragraph{Notation.}
We address multiclass classification with $C$ classes, and write the label set
as the frame of discernment $\Theta=\{1,\dots,C\}$. Training data are $N$
labelled samples $\{(x_i,y_i)\}_{i=1}^{N}$, each with $d$ features and a label
$y_i\in\Theta$. FERL grows a fuzzy tree $T$ with nodes $o$; each internal split
pairs a feature with a fuzzy set of membership $\mu(\cdot)\in[0,1]$, and
memberships multiply along a root-to-$o$ path into a firing strength
$\phi_o(x)\in[0,1]$. Every node stores a consequent class distribution
$p_o\in[0,1]^{C}$, $\sum_c p_o(c)=1$, estimated from the soft weighted samples it covers. At inference, evidence is combined over a node set $A(x)$: all active nodes in the compact variant, or the active leaves in the deep variant. We write $K=\max_x|A(x)|$ for the largest number of evidence sources combined for one input. Each $o\in A(x)$ contributes a Dempster--Shafer mass $m_o$ on $\Theta$, and the
combined, normalised mass $m$ (normaliser $Z$) yields a belief $\mathrm{Bel}(c)$, a plausibility $\mathrm{Pl}(c)$, and an ignorance mass $m(\Theta)$. The pignistic transform $\mathrm{BetP}$ gives the point label $\hat y=\arg\max_c\mathrm{BetP}(c)$, and interval dominance gives the credal (set-valued) prediction $S(x)\subseteq\Theta$, an abstention when $S(x)=\Theta$.
The tree complexity is governed by a rule budget $R$ (terminal leaves, equivalent to rules, with $R\le$ \texttt{max\_rules}) and a total number $J$ of stored nodes, together with a maximum depth $D$, a patience on
non-improving splits, and a minimum gain $\delta$. A learned split
additionally places its threshold at $\theta$ with a linear fuzzy decision boundary of half-width $h$. Features not observed in a root-to-node $o$ path are denoted as ``free'' with respect to $o$.

\subsection{Building the Fuzzy Tree}

FERL grows a classification tree in which every internal split assigns a feature together with one fuzzy set that defines the condition on that feature. This involves two procedures:
\begin{itemize}
	\item The parameters for each split are computed using a bootstrapping procedure that allows us to find fuzzy partitions with statistical robustness with respect to the gain function.
	\item We crop the fuzzy sets' support to avoid spurious firing in non-explored regions, which detects geometric OOD, and store in each node the necessary parameters to do near-OOD detection.
\end{itemize}

\paragraph{Learning the Fuzzy Splits.}
\label{sec:soft-splits}
FERL works with two possible gain functions: the Complete Classification Index (CCI), which measures the gain in correctly classified soft mass obtained by replacing a node with a candidate split; and weighted Gini impurity (WGI). CCI works very well for small tree variants, but it is more costly to compute than WGI, which is the preferred metric for deeper variants. Growth is greedy up to a budget on the number of rules and the depth, which allows us to target different interpretability-performance aims.

For small FERL trees, we pre-fix the possible fuzzy partitions, using intuitive semantics. In these cases, this results in an optimal trade-off between interpretability and performance. For deeper variants, this approach does not scale. So, we repeatedly bootstrap and find the standard Gini-optimal cut in each resample. The resulting empirical distribution of cut locations determines the following fuzzy partition:

\begin{equation}
\mu_{\theta,h}(x_f) \;=\; \min\!\left(1,\ \max\!\left(0,\ \frac{\theta + h - x_f}{2h}\right)\right),
\end{equation}
where $x_f$ is the value of the candidate feature, $\theta$ the median of the per-resample optimal cut locations $\{c_b\}_{b=1}^{B}$, and $h$ is a robust measure of their spread:
\begin{equation}
h \;=\; \gamma\cdot\mathrm{MAD}, \qquad
\mathrm{MAD} \;=\; \mathrm{median}_b\,\lvert c_b-\theta\rvert ,
\end{equation}
with $\gamma>0$ a width multiplier. The other branch receives then the complementary membership $1-\mu_{\theta,h}(x_f)$. The condition is thus fully active for $x_f\le\theta-h$, off for $x_f\ge\theta+h$, and behaves linearly across the interval $[\theta-h,\theta+h]$.

\paragraph{Bounded-Support Memberships and Geometry-Awareness}
A point far outside the training range that activates a leaf at full strength can result in false node confidence. To solve this, we soft-bound the support of each fuzzy membership function to the feature range the node actually observed in training, decaying it to zero beyond a small margin $\epsilon$. This makes out-of-distribution points that fall outside the observed range on some feature have their leaf activations collapse towards zero, which raises the model's ignorance. This mechanism is exact: as shown in the supplementary material, the routing mass lost to the support bounds decomposes node by node, each term naming the split feature responsible, and if every leaf firing vanishes the combined evidential output is provably vacuous. However, this procedure cannot detect a near-OOD point that stays within the observed data manifold. For instance, when the sample belongs to a novel class that overlaps the training set. To solve this, we model the behaviour of \emph{unsplit} features in every node: each of them fits a diagonal Gaussian weighted by the fuzzy memberships to its root-to-node path over the features absent from it. Since the firing $\phi_o$ depends on $x$ only through the features tested on that path, this
is precisely a model of the unsplit features given the split ones:
\begin{equation}
p\big(x_{\mathrm{free}(o)} \,\big|\, \phi_o\big) \;\approx\; \mathcal{N}\!\big(\mu_o,\ \mathrm{diag}(\sigma^2_o)\big).
\end{equation}
The diagonal covariance keeps every unsplit feature's contribution separable, so we can then detect which attributes are anomalous at any path. At test time we compute the firing-weighted average of the diagonal Mahalanobis distance per node. These feature scores show how atypical a sample is where the tree was not looking.

\subsection{Building Evidential Outputs}

In FERL, each fuzzy activation is treated as a piece of evidence: we turn each activated node $o$ into a Dempster--Shafer mass function on the label set $\Theta=\{1,\dots,C\}$: the node places mass $\phi_o(x)\,p_o(c)$ on class $c$ and residual mass $1-\phi_o(x)$ on the whole frame $\Theta$. A node that fires weakly contributes mostly ignorance, and a node that fires strongly commits to its class distribution.  Combining the masses of the active nodes with Dempster's rule of combination returns a belief $\mathrm{Bel}(c)$ and a plausibility $\mathrm{Pl}(c)$ for every class. A point prediction is read from the pignistic transform of the combined mass, and a set-valued prediction is obtained by interval dominance: class $c$ is part of the final prediction when $\mathrm{Pl}(c)\ge\max_{c'}\mathrm{Bel}(c')$. The prediction set contracts to a singleton when one class dominates, expands when several classes are plausible, and covers the whole frame when no evidence reliably separates any of them, which naturally results in an abstention.

An important idea in FERL inference is focusing on \emph{which} nodes to combine. Dempster's rule treats all nodes as independent bodies of evidence, which they are not. In a compact tree, this is not problematic, as the few internal nodes along a path carry complementary evidence. In fact, we find that internal nodes can be used to regularise the overconfident estimations of the leaves. In a deep tree, the nodes are numerous and strongly correlated, and combining all of them drives ignorance to zero. So, in that case, we only combine the leaves. Because a sample's path activates only a handful of leaves, the combined mass retains a meaningful ignorance term, and this term turns out to be a locally adaptive metric for genuinely hard regions.

\begin{figure}[ht!]
		\centering
		\includegraphics[width=0.85\linewidth]{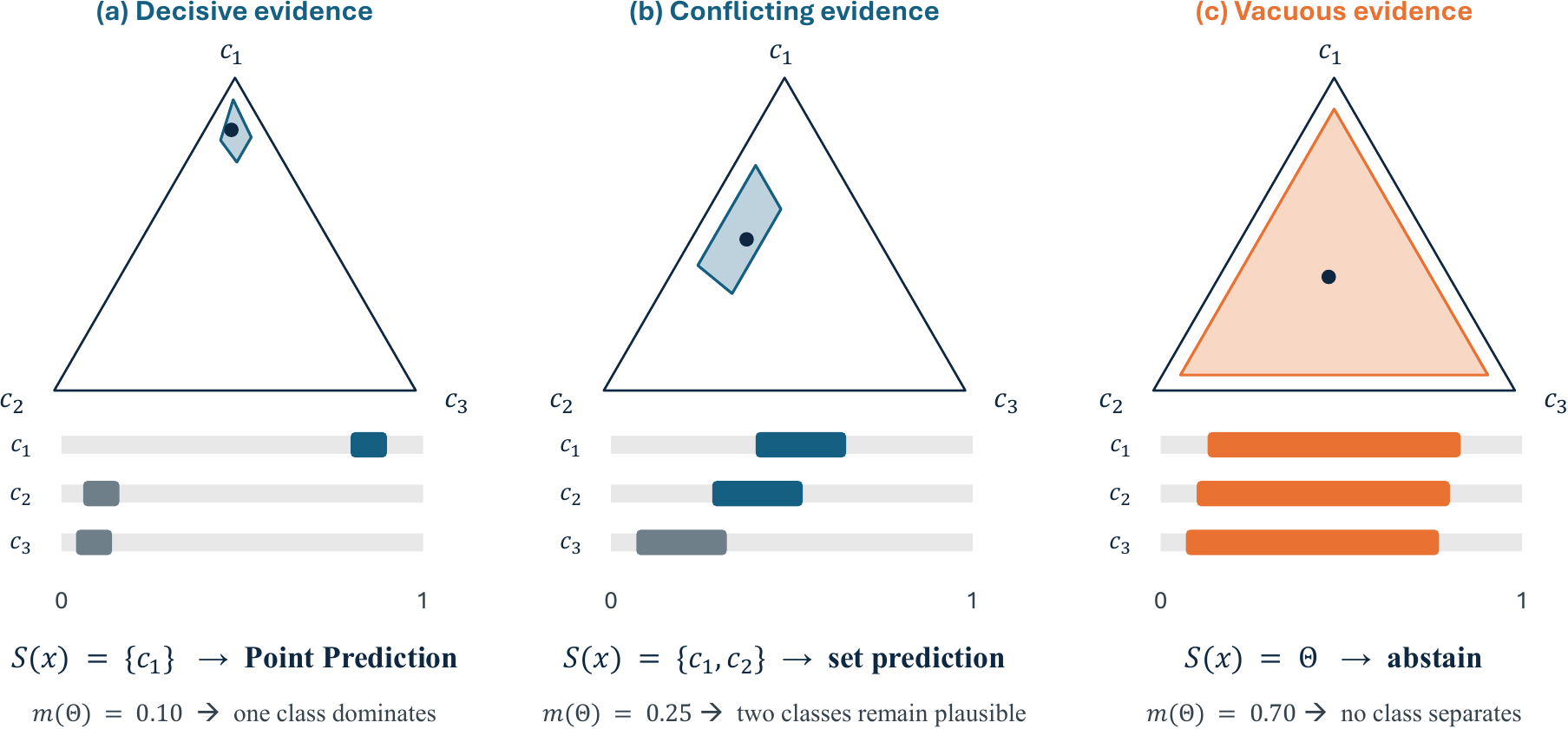}
		\caption{\textbf{FERL evidential output.} FERL naturally expresses doubt about its predictions, computed using Eq.~\ref{eq:decision}. In (a), the belief in the dominant class is clear, so the result is a singleton. In (b), conflicting evidence shows that the plausibility of class c$_2$ is bigger than belief in the dominant class, so the prediction set includes both classes. Finally, in (c) the evidence is not decisive for any class, so FERL abstains.}
		\label{fig:credal_output}
\end{figure}

\paragraph{Analytical Form of the Evidential Output} We write $\phi_o=\phi_o(x)$ for brevity. Each node is a
reliability-discounted Bayesian mass function:
\[
m_o(\{c\})=\phi_o p_o(c), \qquad m_o(\Theta)=1-\phi_o .
\]
 Before normalisation, the combined mass has only the following non-zero terms:
\[
\tilde m(\Theta)=\prod_{o\in A(x)}(1-\phi_o),
\]
\[
\tilde m(\{c\})=
\prod_{o\in A(x)}\left(1-\phi_o+\phi_o p_o(c)\right)
-\prod_{o\in A(x)}(1-\phi_o).
\]
After division by $Z=\tilde m(\Theta)+\sum_c\tilde m(\{c\})$, these quantities
give the normalised mass $m$. The normaliser $Z$ is the total mass Dempster's rule assigns to non-empty focal sets. Dividing by $Z$ redistributes it and restores $\sum_c m(\{c\})+m(\Theta)=1$. Hence $\mathrm{Bel}(c)=m(\{c\})$ and $\mathrm{Pl}(c)=m(\{c\})+m(\Theta)$, and interval dominance simplifies to
\begin{equation} \label{eq:decision}
S(x)=\{c:\ m(\{c\})\geq \max_j m(\{j\})-m(\Theta)\}.
\end{equation}
This credal set prediction chooses classes whose mass lies within the residual ignorance of the best-supported class, and the set prediction expands as the model lacks evidence for the decision (Figure~\ref{fig:credal_output}). This expression also clarifies the dependence problem in trees. Combining an existing singleton class-plus-ignorance mass $m$ with an additional node whose ignorance is $a=1-\phi$ gives
\begin{equation}
\begin{array}{rcl}
m'(\Theta) &=& \frac{m(\Theta)a}{Z'},\\
Z' &=& a+m(\Theta)(1-a)+\sum_c m(\{c\})\phi p(c).
\end{array}
\end{equation}
Since $Z'\geq a$, adding a non-vacuous source cannot increase residual ignorance. In a deep tree, all-node Dempster combination repeatedly discounts ignorance with correlated evidence. This accounts for why we exclude the internal nodes from the final decision in the deep variant, and also suggests the use of selective node combinations or other decision rules in order to obtain good ignorance measures while exploiting the nested structure of the tree. For our experimentation in the main paper, we have used the standard DS combination rule on the leaf nodes, as its empirical performance proved superior to other alternatives. However, we also study in our Supplementary materials other possible alternatives.

\subsection{Analysis of Output Stability}
\label{sec:analysis}

Contrary to crisp trees, FERL output is a continuous function of the input. This allows FERL to make more committed decisions in ``clear'' splits, while being more conservative in ``dubious'' ones. We show how this works in the following proof.
Let $\lambda$ be the largest slope of any fuzzy membership used in the tree, let $K=\max_x|A(x)|$ bound the number of combined nodes, and let $D$ be the maximum depth.

\begin{proposition}[Lipschitz stability]
\label{prop:lipschitz}
The unnormalised evidential masses $\tilde m(\{c\})$ and $\tilde m(\Theta)$ are Lipschitz in $x$ with constant $O(KD\lambda)$. Consequently, on the region $\Omega_\tau=\{x: Z(x)\ge\tau\}$ where the normaliser is bounded below, the pignistic output is Lipschitz,
\[
\|\mathrm{BetP}(\cdot\mid x)-\mathrm{BetP}(\cdot\mid x')\|_1
\;\le\; \frac{\kappa\,K\,D\,\lambda}{\tau}\,\|x-x'\|_\infty ,
\]
for a constant $\kappa$ depending only on the number of classes $C$.
\end{proposition}

\noindent\emph{Proof sketch.}
Each firing strength is a product of at most $D$ single-feature $\lambda$-Lipschitz memberships, and each mass a product and difference of at most $K$ such bounded factors, giving $O(KD\lambda)$; the quotient rule applied to $\mathrm{BetP}(c)=\big(\tilde m(\{c\})+\tilde m(\Theta)/C\big)/Z$ on $Z\ge\tau$ then has this bound. The full proof, and a tighter per-path refinement of the constant, are given in the supplementary material.

The bound shows the impact of tree architecture in the final output: stability improves when fewer evidence sources are combined (smaller $K$), paths are shallower (smaller $D$), and fuzzy memberships are smoother (smaller $\lambda$). The rule budget $R$ controls this indirectly by limiting the available leaves and hence the number of sources that can enter $A(x)$.



\section{Experiments}

\subsection{Setup}

\paragraph{Tabular data.} We evaluate on thirty tabular classification datasets from the UCI repository \cite{uci_repository}, spanning binary and many-class problems with numeric and categorical features, up to 64 features and 19k samples. Results are reported by stratified five-fold cross-validated means. Within each outer training fold, we reserve $25\%$ for conformal calibration and fit every predictor on the remainder. Classification performance is measured in accuracy, and model complexity based on the number of leaves per tree (equivalent to rules in a rulebase). We report the area under the risk--coverage curve (AURC) for selective classification, and for set-valued prediction we report determinacy (fraction of singleton outputs), empirical set coverage, mean set size, and utility-discounted accuracy at the standard $u_{65}$ and $u_{80}$ discount levels \cite{zaffalon2012evaluating}.

\paragraph{Image data.} We use concept-bottleneck models (CBMs)
\cite{koh2020concept} on Caltech--UCSD Birds (CUB)
\cite{wah2011cub} and Animals with Attributes 2 (AwA2)
\cite{xian2019zero}. A neural concept detector maps each image to a vector of
probabilities for human-named attributes, and FERL, CART, and logistic
regression are then fit as alternative classification heads over exactly the
same detector-predicted vectors. This makes FERL rules, confidence, and abstention refer to attributes rather than image pixels. We evaluate accuracy and OOD capabilities in this setting for a standard symbol-extractor--FERL pipeline.

\subsection{Baselines}

\paragraph{FERL configurations.} We use three different variants of FERL: \emph{compact} uses a maximum of 15 rules and has a precomputed fuzzy partition where the conditions map to semantically intuitive concepts (``low'', ``medium'' and ``high''); \emph{medium} uses a maximum of 50 rules; \emph{deep} uses a maximum of 200 rules and combines only the leaves of its tree. The three variants also sit at three points of the stability dial of Section~\ref{sec:analysis}: from compact to deep, larger rule and depth budgets permit larger $K$ and $D$, trading the stability constant of Proposition~\ref{prop:lipschitz} for accuracy. 

\paragraph{Comparison methods.} We compare against modern and classic interpretable trees and rule learners. For crisp tree methods: \textbf{CART}
\cite{breiman1984cart}, \textbf{C4.5} \cite{quinlan1993c45}, and 
\textbf{FIGS} \cite{tan2022figs}. For fuzzy rule-based classifiers, we compare against Fuzzy-rule baselines: \textbf{FURIA} \cite{huhn2009furia} and Fuzzy-UCS (\textbf{FUCS}) that also uses Dempster--Shafer inference \cite{shiraishi2025evidential}. For neural and sampling-based rule learners: \textbf{RRL} \cite{wang2021scalable}, \textbf{RL-Net} \cite{dierckx2023rlnet}, \textbf{NeuRules}
\cite{xu2026neural}. To compare with sampled rule lists we implemented \textbf{SamRuLe}
\cite{pellegrina2024scalable}. We also include a logistic regression (\textbf{LR}) as a non-rule-based interpretable baseline. For dedicated credal and evidential classifiers: the naive credal classifier (\textbf{NCC}) \cite{corani2008ncc}, credal decision trees (\textbf{CDT}) and their C4.5 variant \cite{abellan2003credaltree}, and evidential deep learning (\textbf{EDL}) \cite{sensoy2018edl} (which is non-explainable). For set-valued comparison, \textbf{MLP-Conformal} applies deterministic, tie-conservative split-conformal adaptive prediction sets (APS) at $\alpha=0.1$ \cite{romano2020aps} to the softmax output of a one-hidden-layer, 64-unit multilayer perceptron. For the sake of comparison, we also include random forest (\textbf{RF}) \cite{breiman2001random} and gradient boosting (\textbf{GB}) \cite{chen2016xgboost} as two non-interpretable, highly performant classifiers.

\section{Results}

\subsection{Tabular Data}
\paragraph{Accuracy and Model Complexity}

Table~\ref{tab:frontier} reports accuracy, model size, and selective-risk AURC for all methods. The compact FERL tree attains $77.8$ accuracy with fewer than six rules, already competitive with a full C4.5 tree ($78.2$) that is more than forty times larger, and above it on selective risk. The medium FERL tree using bootstrapping to learn the fuzzy partition lifts accuracy to $79.2$ at about forty rules. It surpasses most more complex rule-learners while matching CART ($79.8$), even when it is an order of magnitude smaller. The deep FERL variant reaches $83.2$, becoming the best tree/rule-based method both in accuracy and AURC. However, the distance with respect to non-interpretable ensemble classifiers is still significant. It is also remarkable that modern rule learners struggle to compete with CART. NeuRules is the strongest one, matching FERL-medium's accuracy ($79.3$ vs.\ $79.2$) with a smaller rule count but at higher selective risk ($14.1$ vs.\ $12.7$ AURC). RRL ($77.9$) and RL-Net are particularly complex for the accuracy they achieve, and while SamRuLe's sampled rule lists stay compact, they fall well below FERL-compact. FUCS reaches $78.59$ only with rule bases in the thousands. Figure~\ref{fig:frontier} shows the evolution of accuracy and model complexity for all the rule learners involved.

\newcommand{\SigNdatasets}{30}
\newcommand{\SigFriedmanAccP}{1e-12}
\newcommand{\SigFriedmanAurcP}{6e-23}
\newcommand{\SigFerlAccRank}{2.97}
\newcommand{\SigFerlAurcRank}{2.57}

A Friedman test across the 30 tabular datasets rejects equal performance among
the thirteen interpretable methods for both accuracy and AURC ($p<0.001$ in
both cases). FERL-deep obtains the best average rank on both metrics (accuracy
rank \SigFerlAccRank, AURC rank \SigFerlAurcRank), while
Figure~\ref{fig:cd} shows the Nemenyi post-hoc comparison for accuracy.
Complementary Holm-corrected Wilcoxon signed-rank tests across datasets find
FERL-deep statistically indistinguishable in accuracy only from LR and better
than every other rule learner (all $p<0.01$). The complete pairwise accuracy
and AURC results, together with a per-dataset analysis, are provided in the
supplementary material.

\paragraph{Training and Inference Cost}
Timing every method in one process on one CPU, and taking medians over the
thirty datasets, FERL-deep fits in $0.26$\,s: among rule learners only CART
($0.002$\,s) and FIGS ($0.17$\,s) fit faster, while NeuRules ($4.8$\,s),
SamRuLe ($12.9$\,s), FUCS ($140.4$\,s) and RL-Net ($146.8$\,s) are one to three
orders of magnitude slower. The margin is largest at inference:
FERL scores a fold in $1.5$\,ms against $16.8$\,ms for FURIA and $744$\,ms for
FUCS, the other DS fuzzy rule learner.

\begin{table}[t]
\centering
\small
\caption{\textbf{Accuracy and selective-risk AURC.} Computed over thirty tabular benchmarks (mean\,$\pm$\,std over datasets, 0--100 scale). ``Size'' is the number of leaves for trees and rules for rule sets. Bold marks the best result for an interpretable model and underlined the best one for a non-interpretable model.}
\label{tab:frontier}
\begin{tabular}{@{}lccc@{}}
\toprule
Method & Acc.\ $\uparrow$ & Size & AURC $\downarrow$\\
\midrule
\multicolumn{4}{@{}l}{\emph{Trees \& rule learners (general)}}\\
CART & 79.76\,$\pm$\,11.69 & 175 & 19.89\,$\pm$\,10.91\\
C4.5 & 78.19\,$\pm$\,12.34 & 254 & 18.07\,$\pm$\,12.20\\
FIGS & 80.63\,$\pm$\,10.67 & 28 & 15.22\,$\pm$\,10.73\\
RRL & 77.95\,$\pm$\,13.71 & 201 & 13.89\,$\pm$\,11.86\\
RL-Net & 74.81\,$\pm$\,13.00 & 195 & 20.80\,$\pm$\,10.51\\
NeuRules & 79.27\,$\pm$\,11.61 & 19 & 14.07\,$\pm$\,11.10\\
SamRuLe & 70.53\,$\pm$\,12.65 & 15 & 20.07\,$\pm$\,10.18\\
\midrule
\multicolumn{4}{@{}l}{\emph{Fuzzy rule-based}}\\
FURIA & 78.51\,$\pm$\,13.57 & 69 & 17.42\,$\pm$\,13.94\\
FUCS (DS) & 78.59\,$\pm$\,13.61 & 1483 & 17.08\,$\pm$\,12.37\\
\midrule
\multicolumn{4}{@{}l}{\emph{Credal \& evidential}}\\
NCC & 77.97\,$\pm$\,12.63 & 290 & 12.07\,$\pm$\,10.31\\
CDT & 76.60\,$\pm$\,11.50 & 125 & 13.90\,$\pm$\,10.21\\
CDT (C4.5) & 76.60\,$\pm$\,11.50 & 125 & 13.90\,$\pm$\,10.21\\
\midrule
\multicolumn{4}{@{}l}{\emph{FERL family (ours)}}\\
FERL-compact & 77.77\,$\pm$\,11.51 & \textbf{6} & 17.33\,$\pm$\,11.61\\
FERL-medium & 79.21\,$\pm$\,9.82 & 43 & 12.72\,$\pm$\,8.54\\
FERL-deep & \textbf{83.23\,$\pm$\,10.27} & 143 & \textbf{9.21\,$\pm$\,8.33}\\
\midrule
\multicolumn{4}{@{}l}{\emph{Linear Model}}\\
Logistic reg. & 81.13\,$\pm$\,12.66 & -- & 10.59\,$\pm$\,11.05\\
\midrule
\multicolumn{4}{@{}l}{\emph{Non-interpretable methods}}\\
RF & \underline{86.62\,$\pm$\,11.10} & - & \underline{6.26\,$\pm$\,7.96}\\
GB & 85.58\,$\pm$\,10.67 & - & 6.60\,$\pm$\,7.36\\
EDL & 84.71\,$\pm$\,11.15 & -- & 7.30\,$\pm$\,8.12\\
\bottomrule
\end{tabular}
\end{table}

\begin{figure}[ht!]
\centering
\includegraphics[width=0.8\linewidth]{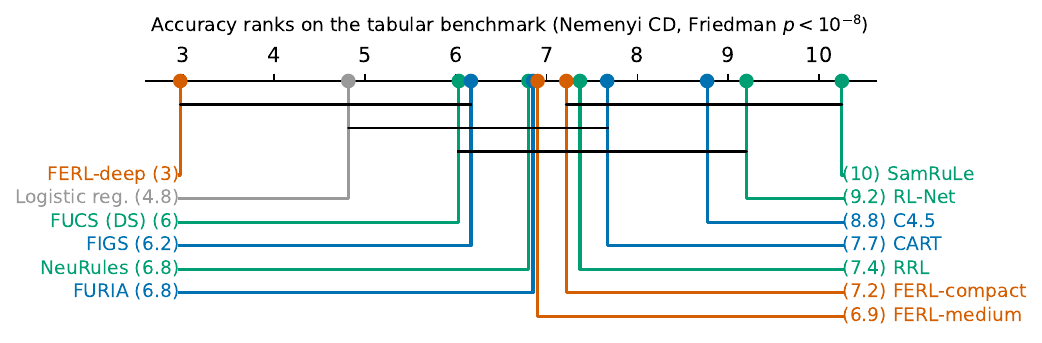}
\caption{\textbf{Critical-difference diagram of accuracy ranks on the tabular
benchmark.} The diagram uses the Nemenyi post-hoc test; methods joined by a bar
are not significantly different. Lower rank is better.}
\label{fig:cd}
\end{figure}

\begin{figure}[ht!]
\centering
\includegraphics[width=0.6\linewidth]{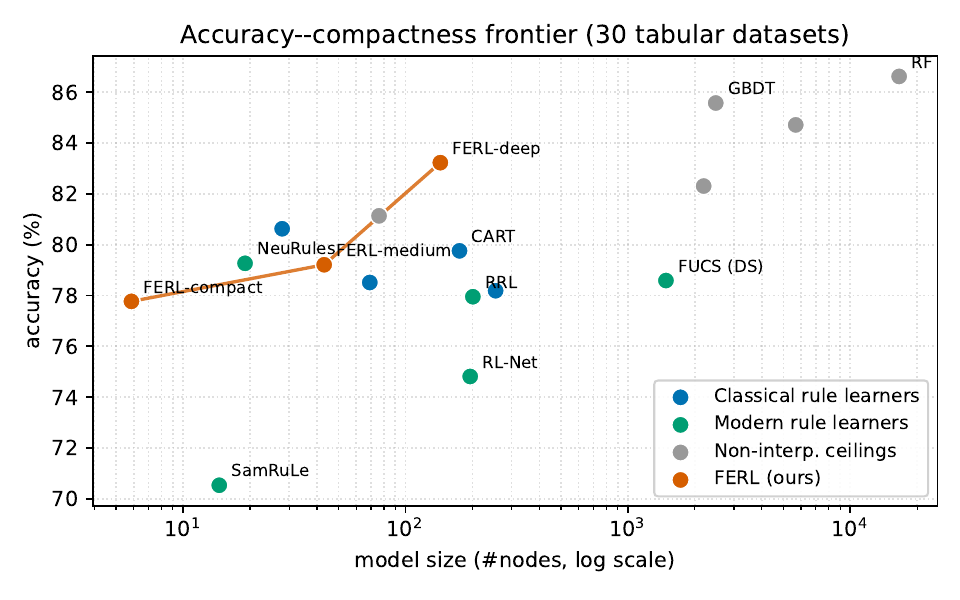}
\caption{\textbf{Accuracy vs.\ model size.} The x-axis corresponds to the log scale of the number of nodes in each tree/rule based classifier, the y axis corresponds to the accuracy obtained by each model.}
\label{fig:frontier}
\end{figure}

\paragraph{Evaluation of Prediction Sets}
\label{sec:results-sets}

Table~\ref{tab:credal} evaluates both native set-valued outputs and post-hoc conformal sets under the utility-discounted accuracy \citep{zaffalon2012evaluating}. FERL-deep obtains the highest utility at both reference discount levels ($u_{65}=0.795$, $u_{80}=0.827$). Best coverage is obtained by MLP-Conformal, which reaches $0.952$, at the cost of the largest prediction sets on average. NCC is the one that obtains the smallest prediction sets. However, this determinism reflects an overconfident bias, as its performance is lower than most rule-learners. Results here show that FERL is capable of very good coverage (close to conformal) while producing significantly smaller sets, without falling into the overconfident estimates of NCC and CDT.
\begin{table}[ht!]
\centering
\small
\setlength{\tabcolsep}{3pt}
\caption{\textbf{Native and conformal set-valued prediction.} Mean\,$\pm$\,std
over the thirty tabular datasets. Determinacy is the singleton fraction,
coverage is empirical, and size is the mean set cardinality. Conformal uses
significance at $\alpha=0.1$; $u_{65}$/$u_{80}$ are utility-discounted
accuracies.}
\label{tab:credal}
\begin{tabular}{@{}lccccc@{}}
\toprule
Method & Determ. & Coverage & Size & $u_{65}$ & $u_{80}$ \\
\midrule
FERL-deep & 0.74\,$\pm$\,0.21 & 0.92\,$\pm$\,0.04 & 1.70\,$\pm$\,1.25 & \textbf{0.80\,$\pm$\,0.14} & \textbf{0.83\,$\pm$\,0.12} \\
FERL-compact & 0.66\,$\pm$\,0.19 & 0.87\,$\pm$\,0.06 & 1.68\,$\pm$\,0.80 & 0.72\,$\pm$\,0.13 & 0.76\,$\pm$\,0.12 \\
\midrule
NCC & \textbf{0.91\,$\pm$\,0.12} & 0.81\,$\pm$\,0.10 & \textbf{1.11\,$\pm$\,0.17} & 0.79\,$\pm$\,0.12 & 0.80\,$\pm$\,0.11 \\
CDT  & 0.88\,$\pm$\,0.08 & 0.82\,$\pm$\,0.08 & 1.19\,$\pm$\,0.19 & 0.78\,$\pm$\,0.11 & 0.79\,$\pm$\,0.10 \\
MLP-Conformal & 0.22\,$\pm$\,0.35 & \textbf{0.95\,$\pm$\,0.05} & 2.17\,$\pm$\,0.84 & 0.63\,$\pm$\,0.16 & 0.74\,$\pm$\,0.14 \\
\bottomrule
\end{tabular}
\end{table}

\paragraph{Out-of-Distribution Detection}

We compare FERL's OOD residual against three dedicated post-hoc OOD detectors: class-conditional Mahalanobis distance \cite{lee2018mahalanobis}, $k$-nearest-neighbour distance \cite{sun2022knnood}, and Isolation Forest
\cite{liu2008isolation}. The experiment setting is leave-one-class-out on the multiclass datasets from the tabular benchmark. In this case, rules are still firing so naive ignorance is uninformative. Table~\ref{tab:ood} shows that FERL-deep's OOD residual signal is marginally better than the strongest dedicated competitor ($77.7$ vs.\ $77.4$ for $k$NN distance) and clearly ahead of Mahalanobis and Isolation Forest, and that it does so with the lowest standard deviation across datasets of any method, i.e.\ it is the most consistent detector of the comparison. FERL-medium is also capable of surpassing class-conditional Mahalanobis and Isolation Forest. The same ranking holds when fixing a threshold. Accepting $95\%$ of retained-class inputs, FERL-deep rejects the largest fraction of novel-class inputs ($36.5\%$), ahead of every dedicated detector.

\begin{table}[ht!]
\centering
\caption{\textbf{OOD-detection: AUROC and rejection.} \emph{Reject} is the fraction of held-out novel-class inputs
rejected at a threshold fixed to accept $95\%$ of retained-class inputs.}
\label{tab:ood}
\begin{tabular}{@{}lcc@{}}
\toprule
Method & AUROC & Reject \\
\midrule
\multicolumn{3}{@{}l}{\emph{FERL}}\\
FERL-compact & 67.41 $\pm$ 7.19 & 24.46 $\pm$ 12.66\\
FERL-medium  & 76.10$\pm$11.26 & 31.82 $\pm$ 16.30\\
FERL-deep    & \textbf{77.66$\pm$10.67} & \textbf{36.47 $\pm$ 16.89}\\
\midrule
\multicolumn{3}{@{}l}{\emph{Native OOD detectors}}\\
Mahalanobis   & 74.42$\pm$13.20 & 30.69 $\pm$ 20.49\\
$k$NN distance     & 77.38$\pm$13.86 & 31.28 $\pm$ 22.10\\
Isolation Forest & 70.33$\pm$11.34 & 25.45 $\pm$ 18.29\\
\midrule
Softmax entropy & 54.18 $\pm$ 7.55 & 4.00 $\pm$ 4.34\\
\bottomrule
\end{tabular}
\end{table}

\subsection{Concept Bottleneck Integration}
We integrate FERL with a CBM by replacing its classification head and training on the symbol detector's output. In the following, we compare it against another tree-based method, CART, and LR, which is the reference model used in CBMs for the explainability of its weights.

\paragraph{Accuracy Results.}
Table~\ref{tab:cbm} reports the accuracy for different classification heads on CUB-20, CUB-200, and AwA2 using the same symbol extractor. We found that per-concept calibration is particularly important on CUB: FERL-deep improves from $77.67\%$ to $86.28\%$ on CUB-20 and from $59.38\%$ to $65.18\%$ on CUB-200. With calibrated concepts it slightly exceeds CART on both tasks, making it the strongest tree-based CBM head in this comparison. Still, standard LR remains the strongest predictor. The nature of this advantage comes from the fact that the LR is effectively using more concepts for each decision. We can see that its advantage disappears when limiting the number of concepts used to the same as FERL. This shows that although the largest LR weights have been typically used as a way to interpret the final decision of the CBM, it is actually a large number of symbols that are required to acquire good performance. On AwA2, all methods perform similarly around $95\%$ accuracy, and the calibration offered no additional gains.

  \begin{table}[t]
	\centering
	\small
	\setlength{\tabcolsep}{4pt}
	\caption{\textbf{Concept-bottleneck accuracy.} ``Raw'' and ``calib.''
		show the effect of per-concept isotonic calibration on the symbol
		extractor. ``Concept-restricted LR'' is an LR head restricted to the
		same number of concepts as the FERL tree.}
	\label{tab:cbm}
	\begin{tabular}{@{}lcc@{}}
		\toprule
		Method & raw & calib.\\
		\midrule
		\multicolumn{3}{@{}l}{\textit{CUB-20}}\\
		FERL-deep              & 77.67\,$\pm$\,4.44 & 86.28\,$\pm$\,1.57\\
		CART                   & 80.32\,$\pm$\,2.23 & 85.57\,$\pm$\,1.29\\
		Concept-restricted LR  & 72.75\,$\pm$\,1.29 & 74.24\,$\pm$\,0.81\\
		LR          & 88.35\,$\pm$\,1.78 & 88.54\,$\pm$\,0.70\\
		\midrule
		\multicolumn{3}{@{}l}{\textit{CUB-200}}\\
		FERL-deep              & 59.38\,$\pm$\,4.57  & 65.18\,$\pm$\,0.77\\
		CART                   & 64.66\,$\pm$\,1.24  & 64.84\,$\pm$\,1.34\\
		Concept-restricted LR  & 49.54\,$\pm$\,24.70 & 58.12\,$\pm$\,1.89\\
		LR         & 70.42\,$\pm$\,1.03  & 70.71\,$\pm$\,1.01\\
		\midrule
		\multicolumn{3}{@{}l}{\textit{AwA2}}\\
		FERL-deep              & 95.08\,$\pm$\,0.11 & 94.91\,$\pm$\,0.14\\
		CART                   & 94.69\,$\pm$\,0.17 & 94.54\,$\pm$\,0.17\\
		Concept-restricted LR  & 95.80\,$\pm$\,0.11 & 95.78\,$\pm$\,0.22\\
		LR          & 95.76\,$\pm$\,0.13 & 95.79\,$\pm$\,0.15\\
		\bottomrule
	\end{tabular}
\end{table}

\paragraph{OOD Detection.}
We next evaluate whether the same concept-level representation supports detection of classes that are entirely unseen at deployment. For each dataset, we partition the classes into five folds and remove the held-out fold before training both the image-to-concept detector and the classification heads. Per-concept isotonic calibration, FERL fitting, and OOD-threshold selection use retained-class data only, and held-out-class test images are used only for final evaluation. Mahalanobis distance, $k$NN distance, and Isolation Forest receive exactly the same calibrated concept vectors as FERL. We first average the five held-out-class folds within each detector seed and then report mean and standard deviation over three detector seeds.

Table~\ref{tab:cbm-open-world-ood} shows that FERL's native residual score is competitive with dedicated OOD detectors without adding a separate detector. On CUB-200, FERL reaches $68.2$ AUROC and $31.0$ AUPR-Out, within $2.3$ and $2.4$ points, respectively, of the strongest dedicated result. On AwA2, it reaches $86.2$ AUROC, within $2.0$ points of $k$NN, while obtaining the best AUPR-Out ($68.3$) and the highest unseen-class rejection at the retained-class validation threshold ($57.2\%$). This operating-point advantage is not uniform: FERL's AwA2 FPR95 is $75.8$, compared with $54.2$--$54.3$ for Mahalanobis and $k$NN. Thus, the residual is also a competitive near-OOD signal at no additional cost within a CBM scheme.

  \begin{table}[ht!]
	\centering
	\footnotesize
	\setlength{\tabcolsep}{2.2pt}
	\caption{\textbf{Detector-class-disjoint open-world CBM detection.}
		Best results per dataset and metric are bold.}
	\label{tab:cbm-open-world-ood}
	\begin{tabular}{@{}lcccc@{}}
		\toprule
		Method & AUROC & AUPR-Out & FPR95 & Reject \\
		\midrule
		\multicolumn{5}{@{}l}{\textbf{CUB-200}} \\
		\addlinespace[1pt]
		FERL-deep
		& 68.2\,$\pm$\,2.4
		& 31.0\,$\pm$\,1.5
		& 82.3\,$\pm$\,2.6
		& 10.2\,$\pm$\,0.3 \\
		\addlinespace[1pt]
		Mahalanobis
		& 70.2\,$\pm$\,2.5
		& 32.3\,$\pm$\,2.4
		& \textbf{73.3\,$\pm$\,3.7}
		& 13.5\,$\pm$\,1.8 \\
		$k$NN distance
		& \textbf{70.5\,$\pm$\,3.0}
		& \textbf{33.4\,$\pm$\,3.1}
		& 74.3\,$\pm$\,5.2
		& \textbf{13.6\,$\pm$\,2.0} \\
		Isolation Forest
		& 58.8\,$\pm$\,1.4
		& 25.6\,$\pm$\,1.5
		& 89.3\,$\pm$\,1.3
		& 8.5\,$\pm$\,1.7 \\
		\midrule
		\multicolumn{5}{@{}l}{\textbf{AwA2}} \\
		\addlinespace[1pt]
		FERL-deep
		& 86.2\,$\pm$\,0.2
		& \textbf{68.3\,$\pm$\,0.5}
		& 75.8\,$\pm$\,1.4
		& \textbf{57.2\,$\pm$\,0.9} \\
		\addlinespace[1pt]
		Mahalanobis
		& 88.0\,$\pm$\,0.8
		& 67.5\,$\pm$\,0.4
		& \textbf{54.2\,$\pm$\,5.4}
		& 56.2\,$\pm$\,0.3 \\
		$k$NN distance
		& \textbf{88.2\,$\pm$\,0.4}
		& 67.7\,$\pm$\,0.4
		& 54.3\,$\pm$\,4.9
		& 53.2\,$\pm$\,1.8 \\
		Isolation Forest
		& 81.8\,$\pm$\,0.4
		& 55.4\,$\pm$\,2.6
		& 78.6\,$\pm$\,1.2
		& 39.1\,$\pm$\,3.7 \\
		\bottomrule
	\end{tabular}
\end{table}

\FloatBarrier
\section{Conclusion}

We presented FERL, a fast and scalable fuzzy rule-learning method in which Dempster--Shafer evidence is the product of the fuzzy memberships that define the rules, and where each node split is computed using bootstrapped samples that create robust fuzzy partitions. FERL provides a high-performance interpretable classifier with set-valued prediction, abstention, and OOD detection, and can target different accuracy--complexity trade-offs through its rule budget. In a 30-dataset tabular benchmark, the medium-sized FERL tree is competitive with classical and recent rule learners while using fewer rules than most of them, and the deep FERL tree surpasses every other single rule learner with statistical significance. Its native residual also matches or closely approaches dedicated OOD detectors without fitting a separate detector. Within a CBM, FERL remains competitive as a classification head and carries the same OOD mechanism into a detector-class-disjoint open-world setting: it stays within $2.3$ AUROC points of the strongest dedicated detector on CUB and AwA2, while attaining the best AwA2 AUPR-Out and unseen-class rejection. These results show that a single concept-level FERL head can provide auditable rules, evidential abstention, and attribute-level novelty diagnostics without an auxiliary uncertainty model.

\bibliography{refs,human}

\clearpage
\appendix
\renewcommand{\headeright}{Supplementary Material}
\section*{Supplementary Material}

\begin{center}
\small
\begin{tabular}{@{}p{0.86\columnwidth}r@{}}
\multicolumn{2}{c}{\bfseries Supplementary Contents}\\[2pt]
\ref{app:construction-guarantees}. Tree Construction Guarantees and a Confidence-Certified Band
    & p.~\pageref{app:construction-guarantees}\\
\ref{app:complexity}. Computational Complexity
    & p.~\pageref{app:complexity}\\
\ref{app:ablations}. Ablation Studies
    & p.~\pageref{app:ablations}\\
\ref{app:tabular-significance}. Statistical Comparison on the Tabular Benchmarks
    & p.~\pageref{app:tabular-significance}\\
\ref{app:frontier-nuance}. Where FERL Wins and Loses: A Per-Dataset Study
    & p.~\pageref{app:frontier-nuance}\\
\ref{app:covariate-shift}. Coverage Under Covariate Shift
    & p.~\pageref{app:covariate-shift}\\
\ref{app:repro}. Reproducibility
    & p.~\pageref{app:repro}\\
\end{tabular}
\end{center}

\section{Tree Construction Guarantees and a Confidence-Certified Band}
\label{app:construction-guarantees}

The learned FERL tree uses the same greedy search pattern as a classical tree,
but its fuzzy routing allows us to characterise its behaviour in different ways than a crisp tree like CART or C4.5. This section discusses these, while reasoning why they are relevant for FERL's behaviour.

\subsection{Guaranteed Gini Improvement}

Consider a node with non-negative incoming sample weights $w_i$, total weight
$W=\sum_i w_i>0$, and class distribution $p$. A complementary fuzzy split has
left memberships $\mu_i\in[0,1]$ and right memberships $1-\mu_i$. Let
$W_L=\sum_i w_i\mu_i$, $W_R=\sum_i w_i(1-\mu_i)$,
$\alpha=W_L/W$, and let $p_L,p_R$ be the corresponding child class
distributions. Write $G(q)=1-\lVert q\rVert_2^2$ for Gini impurity.

\paragraph{Proposition 1 (descent of the construction objective).}
For every complementary fuzzy split with non-empty children,
\begin{equation}
G(p)-\alpha G(p_L)-(1-\alpha)G(p_R)
=\alpha(1-\alpha)\lVert p_L-p_R\rVert_2^2\geq 0.
\label{eq:fuzzy-gini-descent}
\end{equation}
The decrease is strict exactly when both children have positive weight and
their class distributions differ.

\paragraph{Proof.}
Complementarity gives $W_L+W_R=W$ and
$p=\alpha p_L+(1-\alpha)p_R$. Expanding the squared norm of this convex
combination gives
$\alpha\lVert p_L\rVert_2^2+(1-\alpha)\lVert p_R\rVert_2^2-
\lVert p\rVert_2^2=
\alpha(1-\alpha)\lVert p_L-p_R\rVert_2^2$, which is
Equation~\eqref{eq:fuzzy-gini-descent}. \hfill$\square$

\paragraph{Practical Implications.} This property is relevant to guarantee that our splits will always get Gini improvement no matter the split technique used. This is different, for example, from CART, where its splitting method guarantees one-step optimality by checking all possible candidates.

\subsection{Routing-Mass Conservation}

For a test point $x$, let $\phi_o(x)$ be the firing that reaches node $o$,
with $\phi_{\mathrm{root}}(x)=1$. The factor $\gamma_o \in [0,1]$ penalises the membership for being out of the known data manifold. At an internal node, write $\mu_o(x)$ for the left membership and $1-\mu_o(x)$ for the right membership.

\paragraph{Proposition 2 (support conservation in in-distribution samples).}
If every $\gamma$ equals one for sample $x$, then the leaf firings of a finite binary
learned tree form a partition of unity:
\begin{equation}
\sum_{\ell\in\mathcal{L}(T)}\phi_\ell(x)=1.
\label{eq:routing-conservation}
\end{equation}

\paragraph{Proof.}
The two children of $o$ receive
$\phi_o\mu_o$ and $\phi_o(1-\mu_o)$, whose sum is $\phi_o$.
Replacing an internal-node firing by its two child firings therefore preserves
the total. Recursing from the root proves
Equation~\eqref{eq:routing-conservation}. \hfill$\square$

Note that in samples where some values are outside the known data manifold, at least some $\gamma>0$. Then, there is a loss of mass in the routing that has this decomposition:
\begin{equation}
1-\sum_{\ell\in\mathcal{L}(T)}\phi_\ell(x)
=\sum_{o\in\mathcal{I}(T)}\phi_o(x)\bigl(1-\gamma_o(x)\bigr),
\label{eq:support-deficit}
\end{equation}
where $\mathcal{I}(T)$ is the set of internal nodes and $\phi_o$ is the firing
before applying $\gamma_o$. This follows by the same  argument, because the children now sum to $\phi_o\gamma_o$. \paragraph{Practical Implications.} Equation~\eqref{eq:support-deficit} is what makes geometric novelty attributable to a feature: each summand names the node and split feature responsible for the loss of mass.

\subsection{Continuity and Local Prediction Stability}

On an in-support region, we define:
\begin{equation}
P_T(x)=\sum_{\ell\in\mathcal{L}(T)}\phi_\ell(x)p_\ell,
\end{equation}
which is a probability vector. The ramp at node $o$ is
$1/(2h_o)$-Lipschitz. For the $\ell_\infty$ input norm and $\ell_1$ output
norm, define
\begin{equation}
L_T=\max_{\pi\in\mathrm{paths}(T)}\sum_{o\in\pi}\frac{1}{h_o}.
\label{eq:tree-lipschitz}
\end{equation}

\paragraph{Proposition 3 (local robustness certificate).}
On any region where the support gates remain one,
$\lVert P_T(x)-P_T(x')\rVert_1\leq L_T\lVert x-x'\rVert_\infty$.
If $c^*=\arg\max_c P_T(x)_c$ is unique and
\begin{equation}
\Delta(x)=P_T(x)_{c^*}-\max_{c\neq c^*}P_T(x)_c,
\end{equation}
then the predicted class is unchanged for every
\begin{equation}
\lVert x-x'\rVert_\infty<\frac{\Delta(x)}{2L_T}.
\label{eq:local-certificate}
\end{equation}

\paragraph{Proof.}
At a node, the subtree prediction is
$P=\mu P_L+(1-\mu)P_R$. Since probability vectors are at most $2$ apart in
$\ell_1$, the change due to the ramp is bounded by
$2\cdot(1/(2h_o))\lVert x-x'\rVert_\infty$.
\paragraph{Full proof of the pignistic stability bound (Proposition~1 of the main paper).}
Let $K=\max_x|A(x)|$ be the maximum number of evidence sources combined for one
input. Each membership is $\lambda$-Lipschitz and depends on a single feature, so a
firing strength $\phi_o$, a product of at most $D$ of them, has $\ell_1$
gradient norm at most $D\lambda$ (matching the $\ell_\infty$ norm on the input)
and lies in $[0,1]$. The unnormalised masses $\tilde m(\{c\})$ and
$\tilde m(\Theta)$ are products and differences of at most $K$ factors of the
form $1-\phi_o+\phi_o p_o(c)\in[0,1]$; a product of $K$ factors, each bounded by
one and $D\lambda$-Lipschitz, is $KD\lambda$-Lipschitz, hence both masses are
$O(KD\lambda)$-Lipschitz. Now
$\mathrm{BetP}(c)=\big(\tilde m(\{c\})+\tilde m(\Theta)/C\big)/Z$ is a ratio
whose numerator is $O(KD\lambda)$-Lipschitz and bounded in $[0,1]$, and whose
denominator $Z=\sum_c\tilde m(\{c\})+\tilde m(\Theta)$ is itself
$O(CKD\lambda)$-Lipschitz and bounded in $[\tau,1]$ on
$\Omega_\tau=\{x:Z(x)\ge\tau\}$; the quotient rule then bounds each
$\|\nabla\mathrm{BetP}(c)\|_1$ by $O(KD\lambda/\tau)$ up to a $C$-dependent
factor, and summing over the $C$ classes yields the stated bound with the
dependence on $C$ absorbed into $\kappa$. \hfill$\square$

\paragraph{Practical Implications.} This proposition allows us to do two things. First, we can know how robust FERL is to noise. If $L_T$ is small we know that small changes are less likely to affect FERL. Second, we can use this property to know when two samples are going to have the same predictions. The most straightforward use of this proposition would be to build a cache-with-guarantees for FERL in order to save time in inferences for very large datasets or FERL models.

\section{Computational Complexity}
\label{app:complexity}

\begin{algorithm}[!t]
	\caption{FERL: fuzzy-tree growth and evidential inference.}
	\label{alg:ferl}
	\begin{algorithmic}[1]
		\STATE \textbf{Training input:} data $\{(x_i,y_i)\}_{i=1}^{N}$, fuzzy partitions, budget $R$, depth $D$, patience, gain function and min.\ gain value $\delta$
		\STATE initialise $T$ with a single root holding all samples at membership $1$
		\WHILE{$\mathrm{rules}(T)<R$ and best achievable coverage $\ge$ threshold}
		\FORALL{nodes $o$ in $T$ with depth $<D$}
		\STATE score every candidate split of $o$ (feature $\times$ fuzzy set, or a learned threshold $\theta$ with band $h$) by its gain \COMMENT{$\theta$: bootstrap median, $h=\gamma\cdot$MAD}
		\ENDFOR
		\STATE $(o^\star,s^\star,g^\star)\gets$ split with the largest gain
		\IF{$g^\star\le\delta$}
		\STATE increment patience counter; \textbf{break} if exhausted
		\ENDIF
		\STATE install $s^\star$ at $o^\star$; update firings $\phi_o$ and consequents $p_o$
		\ENDWHILE
		\STATE optionally apply cost-complexity pruning
		\STATE \textbf{return} fuzzy rule tree $T$
	\end{algorithmic}
	\rule{\linewidth}{0.4pt}
	\begin{algorithmic}[1]
		\STATE \textbf{Inference input:} sample $x$, tree $T$, inference mode (depending on compact / deep tree)
		\STATE $A(x)\gets$ active leaves if deep, else all active nodes
		\FORALL{$o\in A(x)$}
		\STATE $m_o(\{c\})\gets\phi_o(x)\,p_o(c)$;\quad $m_o(\Theta)\gets 1-\phi_o(x)$
		\ENDFOR
		\STATE $\tilde m(\Theta)\gets\prod_{o\in A(x)}(1-\phi_o)$ \COMMENT{Dempster, commonality form}
		\STATE $\tilde m(\{c\})\gets\prod_{o\in A(x)}\!\big(1-\phi_o+\phi_o p_o(c)\big)-\tilde m(\Theta)$
		\STATE normalise by $Z=\tilde m(\Theta)+\sum_c\tilde m(\{c\})$ to obtain $m$
		\STATE $\mathrm{Bel}(c)\gets m(\{c\})$;\quad $\mathrm{Pl}(c)\gets m(\{c\})+m(\Theta)$
		\STATE $\hat y\gets\arg\max_c\mathrm{BetP}(c)$;\quad native set $S(x)\gets\{c:\mathrm{Pl}(c)\ge\max_{c'}\mathrm{Bel}(c')\}$
		\STATE \textbf{return} $\hat y$; native set $S(x)$; ignorance $m(\Theta)$
	\end{algorithmic}
\end{algorithm}

We analyse the time and space complexity of FERL for both training and
inference. Table~\ref{tab:complexity-notation} separates the rule count reported
in the experiments from the total number of stored nodes and from the number of
nodes used as evidence. We treat one membership evaluation on one feature value
as $O(1)$.

\begin{table}[t]
\centering
\small
\begin{tabular}{cl}
\toprule
Symbol & Meaning \\
\midrule
$N$   & training samples \\
$N_s$ & samples used to score a split \\
$n_s$ & split-scoring subsample cap (default $10{,}000$) \\
$n$   & samples in the inference batch \\
$d$   & features \\
$C$   & classes \\
$R$   & terminal leaves (rules), $R\le$ \texttt{max\_rules} \\
$J$   & total stored nodes \\
$K$   & maximum evidence sources combined for one input \\
$D$   & maximum depth \\
$P_{\max}$ & maximum fuzzy sets on any feature (fixed-partition mode) \\
$B$   & bootstrap resamples for the band width \\
\bottomrule
\end{tabular}
\caption{\textbf{Notation used in the complexity analysis.} If split subsampling is
enabled and $N>n_s$, then
$N_s=\min(N,n_s)$; otherwise $N_s=N$. Under the
default automatic policy, subsampling is enabled only when
$N>50{,}000$. Also $K\le J$; in leaves-only inference, $K\le R$.}
\label{tab:complexity-notation}
\end{table}

\subsection{Tree Growth}

FERL grows greedily: each iteration scores every possible candidate split at every node, executes the single best split, and repeats until the leaf/rule budget $R$, the depth limit $D$, the minimum coverage, or the patience counter is exhausted  (Algorithm~\ref{alg:ferl}). A successful iteration creates one
new node, so there are at most $J-1$ successful node creation steps. We use
$J$ for counts below.

Before tree growth, the $B$ bootstrap resamples are generated once and stored
as multiplicity vectors over the split-scoring observations. The values of each
feature are also sorted once, at a preprocessing cost of
$O(dN_s\log N_s)$. Restricting a precomputed feature ordering to a node's
effective region preserves that ordering, while bootstrap multiplicities change
only the observation weights and not their order. Consequently, neither the
nodes nor the bootstrap replicates require additional sorting.

For each node--feature pair, the cuts associated with the $B$ bootstrap replicates are obtained by sweeping the fixed feature ordering and computing cumulative class weights. Each sweep costs $O(N_sC)$, giving $O((B+1)N_sC)$ per node--feature pair and $O(d(B+1)N_sC)$ per node. The bootstrap-generation cost is $O(BN_s)$. Therefore, the total training cost is

\[
T_{\mathrm{grow}}^{\mathrm{learned}}
=
O\!\left(
dN_s\log N_s
+ BN_s
+ Jd(B+1)N_sC
\right).
\]
For $B\geq1$, this simplifies to
\[
T_{\mathrm{grow}}^{\mathrm{learned}}
=
O\!\left(
dN_s\log N_s
+ JdBN_sC
\right).
\]

\subsection{Space}

The fitted model stores $J$ nodes with a $C$-vector class distribution and split
parameters, plus per-node diagonal statistics over up to $d$ free features for
the OOD score. Including the fixed fuzzy partitions, model space
is therefore $O\!\big(J(C+d)+dP_{\max}\big)$; without the residual statistics it
reduces to $O(JC+dP_{\max})$. Inference allocates the activation matrix at
$O(nJ)$ and output arrays at $O(nC)$.

\subsection{Wall-Clock Runtime}

Table~\ref{tab:runtime} complements the asymptotic analysis with a fresh
wall-clock scaling comparison. We fit every method from scratch on ten
representative datasets using five seeded stratified train/test splits. The
datasets span $178$--$5{,}300$ samples, $2$--$57$ features, and $2$--$7$
classes, and are ordered by the size of their input matrix. FERL-compact and FERL-medium remain below $0.15$ and $0.40$ seconds, respectively, on every dataset. FERL-deep grows more nodes and therefore takes up to $2.65$ seconds. On the two largest input datasets (\emph{segment} and \emph{spambase}), FERL-deep is much faster to fit than C4.5 and FURIA, while CART remains the fastest method overall.
\begin{table*}[t]
\centering
\small
\caption{\textbf{Median fit time in seconds.} Across 5 seeded stratified 70/30 train/test splits. Datasets are ordered by their size.}
\begin{tabular}{@{}lrrrrrrrrrr@{}}
\toprule
 & & & & \multicolumn{3}{c}{FERL fit time (s)} & \multicolumn{4}{c}{Baseline fit time (s)} \\
\cmidrule(lr){5-7}\cmidrule(l){8-11}
Dataset & $N$ & $d$ & $C$ & compact & medium & deep & CART & C4.5 & FIGS & FURIA \\
\midrule
glass & 205 & 9 & 5 & 0.022  & 0.023  & 0.178  & 0.001 & 0.141 & 0.074 & 0.365 \\
wine & 178 & 13 & 3 & 0.012  & 0.015  & 0.262  & 0.001 & 0.065 & 0.009 & 0.315 \\
heart & 270 & 13 & 2 & 0.016  & 0.018  & 0.105  & 0.001 & 0.049 & 0.135 & 0.166 \\
australian & 690 & 14 & 2 & 0.006  & 0.013  & 0.166  & 0.002 & 0.200 & 0.096 & 0.326 \\
banana & 5300 & 2 & 2 & 0.017  & 0.025  & 1.093  & 0.005 & 2.329 & 0.251 & 1.330 \\
ionosphere & 351 & 33 & 2 & 0.023  & 0.053  & 0.756  & 0.003 & 0.527 & 0.039 & 0.714 \\
vehicle & 846 & 18 & 4 & 0.070  & 0.073  & 1.021  & 0.003 & 1.326 & 0.142 & 1.130 \\
german & 1000 & 20 & 2 & 0.006  & 0.010  & 0.161  & 0.002 & 0.493 & 0.191 & 0.362 \\
segment & 2310 & 19 & 7 & 0.141  & 0.394  & 2.014  & 0.008 & 12.742 & 0.199 & 2.641 \\
spambase & 4597 & 57 & 2 & 0.142  & 0.188  & 2.652  & 0.036 & 21.061 & 0.718 & 5.476 \\
\bottomrule
\end{tabular}

\label{tab:runtime}
\end{table*}

\section{Ablation Studies}
\label{app:ablations}

\subsection{Evidence Sources and Combination Rule}
\label{app:decision-rule-ablation}

The evidential decision involves two choices: which tree nodes provide evidence
and how their masses are combined. We isolate both choices without refitting the
tree. First, we compare all non-root nodes with leaves only under Dempster's
rule. Second, we also compare Dempster's rule with Den{\oe}ux's cautious rule \cite{denoeux2008cautious}, which and was
designed for non-distinct evidence sources \cite{denoeux2008cautious}. For a
leaf $o$ with firing $\phi_o$ and consequent $p_o(c)$, its singleton canonical
weight is
\begin{equation}
w_o(c)=\frac{1-\phi_o}
{1-\phi_o+\phi_o p_o(c)}.
\end{equation}
The cautious combination uses $w(c)=\min_o w_o(c)$, so repeated or dependent
support does not compound. The pignistic point decision and interval-dominance
set
$S(x)=\{c:\mathrm{Pl}(c)\geq\max_j\mathrm{Bel}(j)\}$ are otherwise unchanged.

\begin{table*}[t]
\centering
\footnotesize
\setlength{\tabcolsep}{3.5pt}
\caption{\textbf{Evidence-source and combination-rule ablation for FERL-deep.} All read-outs use the same fitted tree. The first two rows isolate all-node versus leaves-only evidence under Dempster's rule; the last two isolate Dempster versus Den{\oe}ux's cautious rule over the same leaves (30 datasets, 5 folds). All entries except mean set size are on a 0--100 scale. Lower AURC/ECE is better; higher accuracy, coverage, and $u_{65}$ is better.}
\label{tab:decision-rule-ablation}
\begin{tabular}{@{}lcccccccc@{}}
\toprule
Evidence/read-out & Acc.\ $\uparrow$ & AURC $\downarrow$ & ECE $\downarrow$ & Determ.\ & Set cov.\ & Size & $u_{65}$ $\uparrow$ & Ign.\\
\midrule
Dempster, all nodes & 80.62 & 10.30 & \textbf{14.88} & 99.88 &80.69 & 1.00 & \textbf{80.65} & 0.07\\
Dempster, leaves &\textbf{83.18} & \textbf{10.14} & 15.83 & 74.26 & 92.31 & 1.70 & 79.50 & 22.93\\
Cautious, leaves & 69.58 & 17.76 & 22.09 & 24.04 & \textbf{96.78} & 3.05 & 59.62 & 37.19\\
\bottomrule
\end{tabular}
\end{table*}

Table~\ref{tab:decision-rule-ablation} shows why FERL-deep uses leaves only.
All-node Dempster combination drives mean ignorance to $0.07\%$ and produces
$99.88\%$ singleton sets, so its set coverage ($80.69\%$) nearly reduces to
point accuracy. Restricting the sources to leaves raises ignorance to $22.93\%$
and set coverage to $92.31\%$, while also increasing accuracy from $80.62\%$ to
$83.18\%$ and slightly lowering AURC. This non-degenerate credal output costs
$0.95$ ECE points and $1.15$ $u_{65}$ points relative to the collapsed
all-node read-out.

The cautious rule is more conservative still: relative to Dempster over the
same leaves, coverage rises from $92.31\%$ to $96.78\%$ and ignorance from
$22.93\%$ to $37.19\%$. However, mean set size grows from $1.70$ to $3.05$,
determinacy falls from $74.26\%$ to $24.04\%$, accuracy falls by $13.60$ points,
and both AURC and ECE worsen. Leaves-only Dempster therefore provides the best
balance here: it removes direct ancestor--descendant dependence,
while cautious combination suppresses useful information.

\subsection{Bounded-Support}
\label{app:bounded-support-ablation}

FERL-deep bounds each membership using the feature range observed at that
node. The gate equals one within the node-local training range and decays to
zero over a margin of one such range on either side. Because this gate is
applied only during inference, it can be switched on or off for the same fitted
tree, isolating the effect of bounded support from tree induction.

\begin{table*}[t]
\centering
\footnotesize
\setlength{\tabcolsep}{4pt}
\caption{\textbf{Fuzzy bounded-support ablation.} The fitted tree is unchanged; the gate is toggled only during inference. ID columns use 30 datasets and 3 folds; geometric-OOD columns use 7 datasets and 2 seeds. The bounded row uses the default margin of one node-local training range. All entries except mean set size are on a 0--100 scale. Higher AUROC is better; $\bar{\Phi}_{\mathrm{OOD}}$ is expected to be small for off-support inputs.}
\label{tab:bounded-support-ablation}
\begin{tabular}{@{}lccccccc@{}}
\toprule
& \multicolumn{3}{c}{In distribution} & \multicolumn{4}{c}{Geometric OOD}\\
\cmidrule(lr){2-4}\cmidrule(lr){5-8}
Support gate & Acc.\ $\uparrow$ & Set cov. & Size & Firing AUROC $\uparrow$ & Ign.\ AUROC $\uparrow$ & $\bar{\Phi}_{\mathrm{ID}}$ & $\bar{\Phi}_{\mathrm{OOD}}$\\
\midrule
Unbounded  & 82.00 & 90.65 & 1.31 & 47.03 & 23.19 & 100.00 & 100.00\\
Bounded & 81.99 & 90.65 & 1.31 & 99.29 & 98.62 & 99.67 & 7.28\\
\bottomrule
\end{tabular}
\end{table*}

Table~\ref{tab:bounded-support-ablation} separates in-distribution behaviour from
geometric-OOD detection. Across 30 datasets and three folds, enabling the gate
changes mean accuracy from $82.00\%$ to $81.99\%$, while set coverage remains
$90.65\%$ and mean set size remains $1.31$. On the seven-dataset geometric-OOD
study, firing-based AUROC rises from $47.03\%$ to $99.29\%$ and
ignorance-based AUROC from $23.19\%$ to $98.62\%$. Mean firing remains
$99.67\%$ in distribution but falls to $7.28\%$ for OOD inputs. Thus, the
bounded gate adds an off-support signal without materially changing the
in-distribution operating point. Note that this result concerns geometric support shift only.
It does not imply detection of semantically novel classes that remain within
the observed feature support.

\subsection{Split Criterion: Weighted Gini versus CCI}
\label{app:cci-gini-ablation}

FERL-deep uses weighted Gini to select learned thresholds, whereas the compact
fixed-partition variants use the Complete Classification Index (CCI). We refit
FERL-deep under each criterion using the same five outer folds, the same
$25\%$ calibration reserve within each outer-training fold, and the same
maximum depth. Only the split criterion changes.

\begin{table}[t]
\centering
\small
\caption{\textbf{Split-criterion ablation for FERL-deep.} Weighted Gini vs. the Complete Classification Index. For the learned-threshold deep tree, CCI wins on only 6 of $30$ datasets and is 0.71 points lower on average, on both the binary and multiclass subsets (Pearson $r{=}-0.07$ between $\Delta$ and class count $C$).}
\label{tab:cci-gini-learned}
\begin{tabular}{@{}lcccc@{}}
\toprule
Subset & \# & Gini & CCI & $\Delta$\\
\midrule
Binary ($C{=}2$) & 18 & \textbf{83.32} & 82.45 & -0.87\\
Multiclass ($C{>}2$) & 12 & \textbf{83.09} & 82.63 & -0.46\\
\midrule
\emph{All} & 30 & \textbf{83.23} & 82.52 & -0.71\\
\bottomrule
\end{tabular}
\end{table}

Weighted Gini obtains $83.23\%$ mean accuracy, compared with $82.52\%$ for CCI.
CCI wins on six datasets, ties on one, and loses on 23. Its mean change is
$-0.87$ points on binary datasets and $-0.46$ points on multiclass datasets.
The advantage of CCI for compact fixed partitions therefore does not hold
to this learned-threshold, depth-12 setting, supporting weighted Gini as the
FERL-deep split metric.

\subsection{Prefixed and Split-Driven Fuzzy Partitions}
\label{app:refit}

FERL can use fixed fuzzy partitions for better interpretability, or when provided by an expert. This subsection tests when each choice is more suitable. For that, we fit a FERL-compact with a fixed partition scheme using trapezoidal fuzzy sets based on the quartiles of the distribution. Then, we freeze its topology and its node consequents, and optimise only the trapezoidal breakpoints by gradient descent on training cross-entropy. We report the change in test accuracy, the change in ECE, and the \emph{partition drift}: the mean per-set breakpoint movement as a fraction of the feature range.

Across all tabular datasets tested, the refit changes test accuracy by a mean of $+0.96$ points and a median of $0.0$ and it changes ECE by $-0.002$ on average. The partitions barely move: mean drift $1.3\%$ of the feature range, median $0.2\%$.  So, for the smallest configuration tested (FERL-compact), the pre-computed fixed partition already sits near a local optimum of the membership parameters: globally refitting them, even under favourable model selection, neither reliably helps accuracy nor improves calibration.

\section{Statistical Comparison on the Tabular Benchmarks}
\label{app:tabular-significance}

The critical-difference diagram in the main paper summarises the global
accuracy ranking. Table~\ref{tab:sig} provides the complementary pairwise
analysis, comparing FERL-deep with each interpretable and rule-learning
baseline across the 30 tabular datasets. The tests pair the five-fold mean for
each method within each dataset; cross-validation folds are not treated as
independent observations. The pairwise results refine this global view:
FERL-deep is statistically indistinguishable in accuracy from LR, but
significantly outperforms every other rule learner.

\begin{table}[h!]
\centering
\small
\caption{\textbf{Wilcoxon signed-rank comparisons on the tabular benchmark.} Pairwise comparison of \emph{FERL-deep} against every interpretable and rule-learning baseline across 30 datasets. Entries are mean paired differences: $\Delta$Acc.\ $>0$ and $\Delta$AURC $<0$ favour FERL. Accuracy $p$-values are Holm-corrected across the comparison family; AURC significance markers use unadjusted $p$-values. $^{*}p<0.05$, $^{**}p<0.01$, $^{***}p<0.001$.}
\label{tab:sig}
\begin{tabular}{@{}lcc@{}}
\toprule
vs.\ baseline & $\Delta$Acc.\ $\uparrow$ & $\Delta$AURC $\downarrow$\\
\midrule
CART & +3.47$^{***}$ & -10.68$^{***}$\\
C4.5 & +5.04$^{***}$ & -8.87$^{***}$\\
FIGS & +2.60$^{***}$ & -6.01$^{***}$\\
Logistic reg. & +2.09 & -1.38\\
FURIA & +4.72$^{***}$ & -8.21$^{***}$\\
FUCS (DS) & +6.03$^{**}$ & -9.17$^{***}$\\
RRL & +5.28$^{**}$ & -4.68$^{***}$\\
RL-Net & +8.42$^{***}$ & -11.59$^{***}$\\
NeuRules & +3.96$^{***}$ & -4.87$^{***}$\\
SamRuLe & +12.70$^{***}$ & -10.86$^{***}$\\
FERL-compact & +5.46$^{***}$ & -8.13$^{***}$\\
FERL-medium & +4.02$^{***}$ & -3.51$^{***}$\\
\bottomrule
\end{tabular}
\end{table}

\section{Where FERL Wins and Loses: A Per-Dataset Study}
\label{app:frontier-nuance}

We examine the aggregate tabular comparison in greater detail by focusing on
FERL-deep, FERL-medium, LR, and FIGS. The Wilcoxon comparison in
Table~\ref{tab:sig} finds no significant accuracy difference between FERL-deep
and LR, despite their different inductive biases. Table~\ref{tab:supp-frontier-nuance}
therefore reports their per-dataset accuracies, together with those of
FERL-medium and FIGS, sorted by the FERL-deep$-$LR gap. This analysis identifies
the datasets on which additive linear reasoning is better suited and those on
which FERL's rule-based nonlinear boundaries are advantageous.

\begin{table*}[t]
\centering
\footnotesize
\caption{\textbf{Per-dataset accuracy on the 30-dataset tabular benchmark, four representative frontier methods.} Mean over the $5$ folds; $n$, $d$ and $C$ are the number of samples, raw features and classes. Rows are sorted by $\Delta = \text{FERL-deep} - \text{LR}$ ($>0$ favours FERL).}
\label{tab:supp-frontier-nuance}
\begin{tabular}{@{}lrrr cccc |r@{}}
\toprule
 & & & & \multicolumn{2}{c}{FERL} & & & \\
Dataset & $n$ & $d$ & $C$ & deep & medium & LR & FIGS & $\Delta$\\
\midrule
banana & 5300 & 2 & 2 & \textbf{89.8} & 77.0 & 57.1 & 89.1 & +32.8\\
vowel & 990 & 13 & 11 & \textbf{79.0} & 69.7 & 65.9 & 64.9 & +13.1\\
ring & 7400 & 20 & 2 & 88.4 & 82.1 & 75.8 & \textbf{88.4} & +12.5\\
phoneme & 5404 & 5 & 2 & \textbf{86.5} & 80.2 & 75.3 & 83.7 & +11.2\\
glass & 214 & 9 & 6 & \textbf{69.8} & 61.0 & 59.5 & 68.8 & +10.2\\
ecoli & 336 & 7 & 8 & \textbf{82.6} & 80.1 & 75.2 & 78.9 & +7.4\\
satimage & 6435 & 36 & 6 & \textbf{87.1} & 80.6 & 80.0 & 84.7 & +7.1\\
magic & 19020 & 10 & 2 & \textbf{86.0} & 82.2 & 79.1 & 84.5 & +6.9\\
contraceptive & 1473 & 9 & 3 & 53.5 & 50.2 & 50.8 & \textbf{53.5} & +2.7\\
ionosphere & 351 & 33 & 2 & \textbf{89.5} & 86.3 & 87.2 & 86.9 & +2.3\\
penbased & 10992 & 16 & 10 & \textbf{96.1} & 88.0 & 94.1 & 87.1 & +2.0\\
bupa & 345 & 6 & 2 & \textbf{68.1} & 67.0 & 66.7 & 62.9 & +1.4\\
australian & 690 & 14 & 2 & \textbf{85.9} & 85.5 & 84.8 & 84.8 & +1.2\\
wine & 178 & 13 & 3 & \textbf{94.4} & 89.3 & 93.3 & 87.1 & +1.1\\
mammographic & 830 & 5 & 2 & \textbf{83.1} & 82.4 & 82.2 & 80.8 & +1.0\\
crx & 653 & 15 & 2 & \textbf{86.4} & 86.4 & 86.1 & 84.5 & +0.3\\
segment & 2310 & 19 & 7 & 95.4 & 93.2 & 95.2 & \textbf{95.7} & +0.2\\
pima & 768 & 8 & 2 & \textbf{76.3} & 74.5 & 76.3 & 74.5 & +0.0\\
wisconsin & 683 & 9 & 2 & 95.8 & 94.4 & \textbf{96.6} & 94.4 & -0.9\\
spectfheart & 267 & 44 & 2 & 75.7 & 75.6 & \textbf{76.8} & 71.5 & -1.1\\
spambase & 4597 & 57 & 2 & 91.6 & 87.1 & \textbf{92.8} & 91.7 & -1.2\\
wdbc & 569 & 30 & 2 & 93.7 & 91.4 & \textbf{95.1} & 92.6 & -1.4\\
german & 1000 & 20 & 2 & 73.3 & 71.4 & \textbf{76.0} & 70.7 & -2.7\\
balance & 625 & 4 & 3 & 83.4 & 76.6 & \textbf{87.0} & 83.4 & -3.7\\
texture & 5500 & 40 & 11 & 93.2 & 87.4 & \textbf{97.5} & 87.5 & -4.3\\
saheart & 462 & 9 & 2 & 66.2 & 70.8 & \textbf{71.2} & 64.3 & -5.0\\
optdigits & 5620 & 64 & 10 & 90.9 & 77.5 & \textbf{96.1} & 82.8 & -5.2\\
heart & 270 & 13 & 2 & 77.8 & 76.7 & \textbf{83.3} & 74.8 & -5.6\\
vehicle & 846 & 18 & 4 & 71.9 & 68.3 & \textbf{79.6} & 70.6 & -7.7\\
twonorm & 7400 & 20 & 2 & 85.8 & 83.4 & \textbf{97.7} & 93.5 & -11.9\\
\midrule
Mean & & & & 83.2 & 79.2 & 81.1 & 80.6 & +2.1\\
\bottomrule
\end{tabular}
\end{table*}

\paragraph{The delta between FERL-deep\,/\,LR is structured.}
FERL-deep wins on $17$ datasets, ties on $1$ and loses on $12$; the raw
Wilcoxon $p$ is $0.31$. But the sign of the gap is far from random with respect to dataset geometry. The per-dataset advantage  $\Delta=\text{FERL-deep}-\text{LR}$ correlates \emph{negatively} with input dimensionality (Spearman $\rho=-0.37$, $p=0.05$): FERL's edge is largest on the low-dimensional problems at the top of the table and turns into a deficit on the high-dimensional ones at the bottom. It is essentially uncorrelated with sample size ($\rho=0.11$) and with the number of classes ($\rho=0.09$; multiclass mean $\Delta=+1.9$ vs.\ binary $+2.2$, Mann--Whitney $p=0.69$), so the split is about the \emph{shape} of the decision boundary, not the size of the task. This also creates room for improvement in FERL, for example, by doing feature selection beforehand.

\paragraph{FERL wins where the boundary is curved and low-dimensional.}
The top of the table is dominated by problems whose optimal boundary is
strongly nonlinear in a handful of features. The extreme case is
\emph{banana} ($d{=}2$): its two interleaved crescent clusters have no useful linear separator, so LR barely clears the base rate ($57.1$) while FERL-deep reaches $89.8$. The same mechanism drives \emph{ring} ($+12.5$; nested-sphere, quadratic boundary), \emph{phoneme} ($+11.2$), and the small nonlinear multiclass sets \emph{vowel} ($+13.1$), \emph{glass} ($+10.2$) and \emph{ecoli} ($+7.4$), where soft membership lets a single feature contribute graded evidence to several classes at once. 

\paragraph{LR wins on oblique, near-linearly-separable, high-dimensional data.}
The reverse also applies. The archetype is \emph{twonorm} ($\Delta{=}-11.9$): its optimal boundary is a single oblique hyperplane in $20$ dimensions, exactly what LR represents natively ($97.7$) and exactly what an axis-parallel partition can only approximate as a high-variance ``staircase''. The remaining FERL losses share this pattern --- \emph{vehicle} ($-7.7$), \emph{optdigits} ($-5.2$), \emph{texture} ($-4.3$). These are not FERL-specific errors, but an intrinsic limitation of any rule learning algorithm.

\newcommand{\ShiftLevel}{1}
\newcommand{\ShiftConformalCoverage}{55.0\%}
\newcommand{\ShiftNativeCoverage}{77.3\%}
\newcommand{\ShiftNativeSetSize}{4.10}
\newcommand{\ShiftConformalSetSize}{1.09}

\section{Coverage Under Covariate Shift}
\label{app:covariate-shift}

Here, we test whether the prediction sets react when the test distribution moves away from the calibration distribution, and compare the usefulness of the plausibility score of FERL's decision against the baseline $1-\hat p_y$. FERL DS-conformal thresholds true-class DS plausibility on the calibration fold, while global split conformal thresholds $1-\hat p_y$ on the same fold. Both wrap the same trained FERL predictor and are calibrated in-distribution at $\alpha=0.1$.

Across seven benchmarks and three seeds, we perturb every test feature independently as $x'_f=x_f+k\sigma_f\epsilon_f$, where $\sigma_f$ is its training-set standard deviation, $\epsilon_f\sim\mathcal{N}(0,1)$, and
$k\in\{0,0.25,0.5,1,1.5,2\}$. At $k=0$, both the FERL DS-conformal plausibility set and global split conformal attain approximately $90\%$ coverage. At $k=\ShiftLevel$, predictive accuracy falls to $50.2\%$: global conformal coverage falls to \ShiftConformalCoverage{} while its mean set size remains nearly fixed at \ShiftConformalSetSize, whereas FERL retains \ShiftNativeCoverage{} coverage by widening its mean set from $1.35$ to \ShiftNativeSetSize{} classes (Figure~\ref{fig:shift-coverage}).

\begin{figure*}[t]
\centering
\includegraphics[width=.8\textwidth]{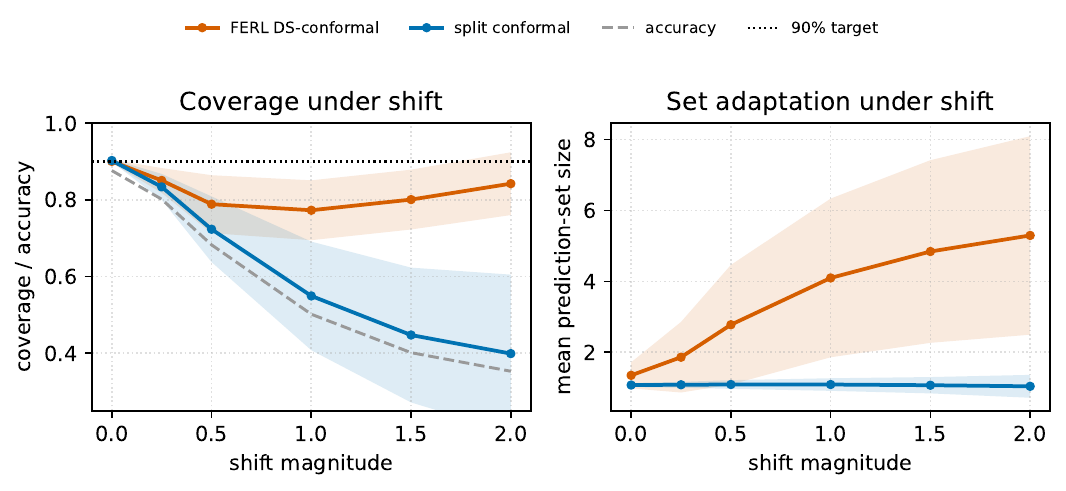}
\caption{\textbf{Prediction-set behaviour under graded covariate shift.}
Mean over seven tabular benchmarks after averaging three seeds per benchmark;
shaded bands are $1.96$ times the standard error across benchmarks. Both
methods begin at the $90\%$ target, but only the FERL DS-conformal plausibility
set expands as the shift grows.}
\label{fig:shift-coverage}
\end{figure*}

We also evaluate a milder feature-domain shift on ten benchmarks: training uses the lower half of a held-out feature's range and testing uses its upper half, after removing that feature from the predictors. FERL attains $69.7\%$ shifted coverage versus $66.7\%$ for global conformal and wins on eight of ten datasets, with mean set sizes $1.36$ and $1.03$, respectively.

\section{Reproducibility}
\label{app:repro}

\paragraph{Tabular benchmarks.}
We use thirty datasets from the KEEL repository \cite{alcala2011keel}: They span $146$--$19{,}020$
instances, $6$--$64$ features, and $2$--$11$ classes. Categorical features are
one-hot encoded and feature normalisation is applied.
Every result is a stratified five-fold cross-validation. Tables report the mean
and standard deviation over the thirty datasets, all on the $0$--$100$ scale.
Table~\ref{tab:datasets} lists them alongside their number of instances, input features (after one-hot encoding of categoricals) and classes.

\begin{table}[t]
\centering
\small
\caption{The thirty tabular benchmark datasets used in the main results, with number of instances, input features (after one-hot encoding of categoricals) and classes.}
\label{tab:datasets}
\begin{tabular}{@{}lrrr@{}}
\toprule
Dataset & \#Samples & \#Features & \#Classes\\
\midrule
australian & 690 & 14 & 2\\
balance & 625 & 4 & 3\\
banana & 5300 & 2 & 2\\
bupa & 345 & 6 & 2\\
contraceptive & 1473 & 9 & 3\\
crx & 653 & 15 & 2\\
ecoli & 336 & 7 & 8\\
german & 1000 & 20 & 2\\
glass & 214 & 9 & 6\\
heart & 270 & 13 & 2\\
ionosphere & 351 & 33 & 2\\
magic & 19020 & 10 & 2\\
mammographic & 830 & 5 & 2\\
optdigits & 5620 & 64 & 10\\
penbased & 10992 & 16 & 10\\
phoneme & 5404 & 5 & 2\\
pima & 768 & 8 & 2\\
ring & 7400 & 20 & 2\\
saheart & 462 & 9 & 2\\
satimage & 6435 & 36 & 6\\
segment & 2310 & 19 & 7\\
spambase & 4597 & 57 & 2\\
spectfheart & 267 & 44 & 2\\
texture & 5500 & 40 & 11\\
twonorm & 7400 & 20 & 2\\
vehicle & 846 & 18 & 4\\
vowel & 990 & 13 & 11\\
wdbc & 569 & 30 & 2\\
wine & 178 & 13 & 3\\
wisconsin & 683 & 9 & 2\\
\bottomrule
\end{tabular}
\end{table}

\paragraph{Concept-bottleneck benchmarks.}
\emph{CUB}: Caltech--UCSD Birds \cite{wah2011cub} with the concepts of \citet{koh2020concept} over $200$ classes. Symbol extractor follows the same architecture as in \citet{koh2020concept}. \emph{AwA2}: Animals with Attributes~2 \cite{xian2018zeroshot} with $85$ binary predicates
over $50$ classes: we use all classes (not the zero-shot split) with a per-image stratified $60/20/20$ partition.Predicted concepts come from a detector we train ourselves: an ImageNet-initialised ResNet-50 with an $85$-way sigmoid head, optimised with BCE for $15$ epochs (Adam, learning rate $10^{-4}$, batch $64$, cosine schedule).

\paragraph{FERL configurations.}
\emph{FERL-compact} uses quantile three-set partitions (implemented as in \citep{fumanal2024ex}), with soft inference and the consistent-CCI split criterion. \emph{FERL-deep} and \emph{FERL-medium} replaces the fixed partition with learned-threshold splits (bootstrap with $25$ resamples) and grows to at most $50$ and $150$ rules, respectively, at maximum depth $5$ and $12$ respectively as well. 

\paragraph{Evaluation protocol.}
Per-concept calibration is isotonic regression fit on the validation split.
Conformal sets use split conformal at $\alpha=0.1$. Risk-controlled abstention calibrates the acceptance threshold with a Bonferroni-corrected Clopper--Pearson bound at $\delta=0.05$. Selective-risk AURC integrates risk over coverage. Utility-discounted accuracy follows \citet{zaffalon2012evaluating}.

\paragraph{Compute and code.}
Symbol Extractors are trained on a single NVIDIA GTX~1080Ti. FERL experiments run on a conventional CPU. Code, configurations, and scripts will be regenerated once the paper is published.

\end{document}